\documentclass[acmlarge]{modified}

\AtBeginDocument{%
  \providecommand\BibTeX{{%
    \normalfont B\kern-0.5em{\scshape i\kern-0.25em b}\kern-0.8em\TeX}}}

\newif\ifCOM
\COMfalse

\begin{document}

\title{On the speed of uniform convergence in Mercer's theorem}

\author{Rustem Takhanov}
\email{rustem.takhanov@nu.edu.kz}
\orcid{0000-0001-7405-8254}
\authornotemark[1]
\affiliation{%
  \institution{School of Sciences and Humanities}
  \streetaddress{53 Kabanbay Batyr Ave}
  \city{Nur-Sultan city}
  \country{Republic of Kazakhstan}
  \postcode{010000}
}

\renewcommand{\shortauthors}{Takhanov}

\begin{abstract}
The classical Mercer's theorem claims that a continuous positive definite kernel $K({\mathbf x}, {\mathbf y})$ on a compact set can be represented as $\sum_{i=1}^\infty \lambda_i\phi_i({\mathbf x})\phi_i({\mathbf y})$ where $\{(\lambda_i,\phi_i)\}$  are eigenvalue-eigenvector pairs of the corresponding integral operator. This infinite representation is known to converge uniformly to the kernel $K$. We estimate the speed of this convergence in terms of the decay rate of eigenvalues and demonstrate that for $2m$ times differentiable kernels the first $N$ terms of the series approximate $K$ as $\mathcal{O}\big((\sum_{i=N+1}^\infty\lambda_i)^{\frac{m}{m+n}}\big)$ or $\mathcal{O}\big((\sum_{i=N+1}^\infty\lambda^2_i)^{\frac{m}{2m+n}}\big)$. Finally, we demonstrate some applications of our results to a spectral charaterization of integral operators with continuous roots and other powers.
\end{abstract}


\maketitle
\section{Introduction}
\label{introduction}
Mercer kernels play an important role in machine learning and is a mathematical basis of such techniques as kernel density estimation and spline models~\cite{WahbaGrace}, Support Vector Machines~\cite{steinwart2008support}, kernel principal components analysis~\cite{KernelPrincipal}, regularization of neural networks~\cite{takhanovDimension} and many others. According to Aronszajn's theorem, any Mercer kernel induces a reproducing kernel Hilbert space (RKHS) and vice versa, any RKHS corresponds to a kernel. A relationship between the latter two notions is decribed in the classical Mercer's theorem. A goal of this note is to refine this theorem and give some estimates on the speed of uniform convergence stated in it.

Let $\boldsymbol{\Omega}\subseteq {\mathbb R}^n$ be a compact set, $K: \boldsymbol{\Omega}\times \boldsymbol{\Omega}\to {\mathbb R}$ be a continuous Mercer kernel~\cite{konig1986eigenvalue} and $L_p(\boldsymbol{\Omega}), p\geq 1$ be a space of real-valued functions $f$ on $\boldsymbol{\Omega}$ with $\|f\|_{L_p(\boldsymbol{\Omega})} = (\int_{\boldsymbol{\Omega}} |f({\mathbf x})|^p d{\mathbf x})^{1/p}$.  
Let ${\rm O}_K: L_2(\boldsymbol{\Omega})\to L_2(\boldsymbol{\Omega})$ be defined by ${\rm O}_K[\phi]({\mathbf x}) = \int_{\boldsymbol{\Omega}} K({\mathbf x},{\mathbf y})\phi({\mathbf y})d{\mathbf y}$. By $C(\boldsymbol{\Omega})$ we denote a space of continuous functions. 
From Mercer's theorem we have that there is an orthonormal basis $\{\psi_i({\mathbf x})\}_{i=0}^\infty$ in $L_2(\boldsymbol{\Omega})$ such that ${\rm O}_{K}[\psi_i]=\lambda'_i\psi_i$. Some of eigenvalues of ${\rm O}_K$ can be equal to zero, therefore, let us assume that natural numbers $i_1<i_2<\cdots$ are such that $\{\lambda'_{i_j}\}_{j=1}^\infty$ is a set of positive eigenvalues, and we denote $\lambda_{j} = \lambda'_{i_j}$ and $\phi_j = \psi_{i_j}$, $j\in {\mathbb N}$.
 It is well-known that $\{\phi_i({\mathbf x})\}_{i=0}^\infty\subseteq C(\boldsymbol{\Omega})$ and $L^N_K = \|K({\mathbf x},{\mathbf y})-\sum_{i=1}^N \lambda_i\phi_i({\mathbf x})\phi_i({\mathbf y})\|^2_{L_2(\boldsymbol{\Omega}\times\boldsymbol{\Omega})} = \sum_{i=N+1}^\infty\lambda_i^2$.  Analogously, for diagonal elements we have $S^N_K = \|K({\mathbf x},{\mathbf x})-\sum_{i=1}^N \lambda_i\phi_i({\mathbf x})^2\|_{L_1(\boldsymbol{\Omega})} = \sum_{i=N+1}^\infty\lambda_i$. Thus, the behaviour of eigenvalues completely characterizes the speed of convergence of $\sum_{i=1}^N \lambda_i\phi_i({\mathbf x})\phi_i({\mathbf y})$ to $K$ in $L_2(\boldsymbol{\Omega}\times\boldsymbol{\Omega})$ and of $\sum_{i=1}^N \lambda_i\phi_i({\mathbf x})^2$ to $K({\mathbf x},{\mathbf x})$ in $L_1(\boldsymbol{\Omega})$.
For the supremum norm, Mercer's theorem implies only the uniform convergence, i.e.
$$
C^N_K = \sup_{{\mathbf x},{\mathbf y}\in \boldsymbol{\Omega}}|K({\mathbf x},{\mathbf y})-\sum_{i=1}^N \lambda_i\phi_i({\mathbf x})\phi_i({\mathbf y})|\to 0
$$
as $N\to \infty$.
We are interested in upper bounds on $C^N_K$.

For $\alpha = (\alpha_1, \cdots, \alpha_n)\in ({\mathbb N}\cup \{0\})^n$, $|\alpha|$ denotes $\sum_{i=1}^n\alpha_i$, $\partial^\alpha_{\mathbf x}f({\mathbf x})$ denotes $\frac{\partial^{|\alpha|}f({\mathbf x})}{\partial x^{\alpha_1}_1\cdots \partial x^{\alpha_n}_n}$. The symbol  $C^m(\boldsymbol{\Omega})$ denotes a set of functions $f:\boldsymbol{\Omega}\to {\mathbb R}$ such that $\partial^\alpha_{\mathbf x}f\in C(\boldsymbol{\Omega})$ for $|\alpha|\leq m$.
We prove the following theorems.
\begin{theorem}\label{GN-theorem} Let $\boldsymbol{\Omega}$ have a Lipschitz boundary, $K\in C^{2m}(\boldsymbol{\Omega}\times \boldsymbol{\Omega})$ and $ p>\frac{n}{m}$, $p\geq 1$. Then,
\begin{equation}
\begin{split}
C_K^N\leq
C^m_{\boldsymbol{\Omega}, p} \max_{\alpha: |\alpha|=m}(\sum_{\beta\leq \alpha}{\alpha \choose \beta} \|\sqrt{D_\beta D_{\alpha-\beta}}\|_{L_p(\boldsymbol{\Omega})})^\theta \big(\sum_{i=N+1}^\infty\lambda_i\big)^{1-\theta} + \\
C^m_{\boldsymbol{\Omega}, p}\sum_{i=N+1}^\infty\lambda_i
\end{split}
\end{equation}
where $D_\alpha ({\mathbf x}) = \partial^\alpha_{\mathbf x}\partial^\alpha_{\mathbf y}K({\mathbf x}, {\mathbf y})|_{{\mathbf y}={\mathbf x}}$, $\theta=(1+\frac{m}{n}-\frac{1}{p})^{-1}$ and $$C^m_{\boldsymbol{\Omega}, p} = \sup_{u\in L_1\cap L_p, u\ne 0}\frac{\|u\|_{L_\infty( \boldsymbol{\Omega})}}{ \|u\|^{1-\theta}_{L_1( \boldsymbol{\Omega})}\cdot \|D^m u\|^{\theta}_{L_p( \boldsymbol{\Omega})}+\|u\|_{L_1( \boldsymbol{\Omega})}}$$
is an optimal constant in the Gagliardo-Nirenberg inequality for the domain $\boldsymbol{\Omega}$.
\end{theorem}
Note that in the latter theorem one can set $p=+\infty$ and obtain that $C^N_K = \mathcal{O}\big((\sum_{i=N+1}^\infty\lambda_i)^{\frac{m}{m+n}}\big)$. Thus,  infinitely differentiable kernels satisfy  $C^N_K = \mathcal{O}\big((\sum_{i=N+1}^\infty\lambda_i)^{1-\varepsilon}\big)$ for any $\varepsilon>0$.

\begin{theorem}\label{GN-theorem2} Let $\boldsymbol{\Omega}$ have a Lipschitz boundary, $K\in C^{2m}(\boldsymbol{\Omega}\times \boldsymbol{\Omega})$ and $ p>\frac{n}{m}$, $p\geq 1$. Then,
\begin{equation}
\begin{split}
C_K^N \leq  D^m_{\boldsymbol{\Omega}, p} (\sum_{i=N+1}^\infty \lambda_i^2)^{(1-\theta)/2}\cdot  \max_{|\alpha|+|\beta|=m} \|\sqrt{D_\alpha }\|^\theta_{L_p}\|\sqrt{D_{\beta}}\|^\theta_{L_p}+\\
D^m_{\boldsymbol{\Omega}, p}(\sum_{i=N+1}^\infty \lambda_i^2)^{1/2}
\end{split}
\end{equation}
where $\theta =(1+\frac{2m}{n}-\frac{2}{p})^{-1}$ and $$D^m_{\boldsymbol{\Omega}, p} = \sup_{u\in L_2\cap L_p, u\ne 0}\frac{\|u\|_{L_\infty( \boldsymbol{\Omega}\times \boldsymbol{\Omega})}}{ \|u\|^{1-\theta}_{L_2( \boldsymbol{\Omega}\times \boldsymbol{\Omega})}\cdot \|D^m u\|^{\theta}_{L_p( \boldsymbol{\Omega}\times \boldsymbol{\Omega})}+\|u\|_{L_2( \boldsymbol{\Omega}\times \boldsymbol{\Omega})}}$$
is an optimal constant in the Gagliardo-Nirenberg inequality for the domain $\boldsymbol{\Omega}\times \boldsymbol{\Omega}$.
\end{theorem}
For $p=+\infty$ we have $C^N_K = \mathcal{O}\big((\sum_{i=N+1}^\infty\lambda^2_i)^{\frac{m}{2m+n}}\big)$. 
For infinitely differentiable kernels, the latter implies  $C^N_K = \mathcal{O}\big((\sum_{i=N+1}^\infty\lambda^2_i)^{0.5-\varepsilon}\big)$ for any $\varepsilon>0$.

\section{Proof of the main theorem}
Let $\mathcal{H}_{K}$ be a reproducing kernel Hilbert space (RKHS) defined by $K$. This space is a completion of the span of $\{K({\mathbf x}, \cdot)\mid {\mathbf x}\in \boldsymbol{\Omega}\}$ with the inner product $\langle K({\mathbf x}, \cdot), K({\mathbf y}, \cdot) \rangle_{\mathcal{H}_{K}} = K({\mathbf x}, {\mathbf y})$. Also, it can be  characterized by the following proposition, which is equivalent to Theorem 4.12 from~\cite{cucker_zhou_2007} and whose original version can be found in~\cite{Cucker2001OnTM}.
\begin{proposition}[\cite{Cucker2001OnTM,cucker_zhou_2007}] \label{smale} Let $\{\lambda_i\}_{i=1}^\infty$ be the set of all positive eigenvalues of ${\rm O}_K$ (counting multiplicities) with corresponding orthogonal unit eigenvectors $\{\phi_i\}_{i=1}^\infty$. Then, $\mathcal{H}_{K}$ equals 
$${\rm O}^{1/2}_{K}[L_2(\boldsymbol{\Omega})] = \{\sum_{i=1}^\infty a_i\phi_i \mid \big[\frac{a_i}{\sqrt{\lambda_i}}\big]_{i=1}^\infty\in l^2\}\subseteq C(\boldsymbol{\Omega})$$ with the inner product
$\langle \sum_{i=1}^\infty a_i\phi_i , \sum_{i=1}^\infty b_i\phi_i \rangle_{\mathcal{H}_{K}} = \sum_{i=1}^\infty \frac{a_ib_i}{\lambda_i}$. For any $f\in \mathcal{H}_{K}$, $$\|f\|_{L_\infty(\boldsymbol{\Omega})}\leq C_K\|f\|_{\mathcal{H}_{K}},$$
where $C_K = \sqrt{\max\limits_{x,y\in \boldsymbol{\Omega}} K(x,y)}$.
\end{proposition}
We will use that proposition throughout our proof.

For any $f\in C(\boldsymbol{\Omega})$, an internal point ${\mathbf x}\in \boldsymbol{\Omega}$, ${\mathbf h}\in {\mathbb R}^n$ and $\alpha\in ({\mathbb N}\cup \{0\})^n$, let us denote
$$
\delta^\alpha_{\mathbf h}[f]({\mathbf x}) = \sum_{\beta: \beta\leq \alpha} (-1)^{|\alpha|-|\beta|} {\alpha \choose \beta} f(x_1 + \beta_1 h_1,\cdots, x_n + \beta_n h_n) 
$$
where $ {\alpha \choose \beta} = \prod\limits_{i=1}^n {\alpha_i \choose \beta_i} $ and $\beta\leq\alpha$ denotes $\beta_i\leq \alpha_i, i=1,\cdots,n$.
For a kernel $K\in C(\boldsymbol{\Omega}\times \boldsymbol{\Omega})$, we have
\begin{equation}\label{2nd}
\begin{split}
\delta^{(\alpha,\alpha)}_{({\mathbf h},{\mathbf h}')}[K]({\mathbf x},{\mathbf x}) = \sum_{\beta,\beta': \beta\leq \alpha,\beta'\leq \alpha} (-1)^{|\beta|+|\beta'|} {\alpha \choose \beta}{\alpha \choose \beta'} K(x_1 + \beta_1 h_1,\cdots, \\ x_n + \beta_n h_n, 
x_1 + \beta'_1 h_1,\cdots, x_n + \beta'_n h_n) 
\end{split}
\end{equation}
If $ \partial^\alpha_{\mathbf x}\partial^\alpha_{\mathbf y}K({\mathbf x}, {\mathbf y})$ exists, let us denote
$$
D_\alpha ({\mathbf x}) = \partial^\alpha_{\mathbf x}\partial^\alpha_{\mathbf y}K({\mathbf x}, {\mathbf y})|_{{\mathbf y}={\mathbf x}}.
$$
Note that $\delta^\alpha_{\mathbf h}$ is a  finite difference operator of a higher order. Its well-known property is given below. 
\begin{proposition}\label{finite-diff} If $f \in C^{|\alpha|}(\boldsymbol{\Omega})$, then $\delta^\alpha_{\mathbf h} [f]({\mathbf x}) = \partial^\alpha_{\mathbf x} f({\mathbf x}){\mathbf h}^{\alpha} + r({\mathbf x}, {\mathbf h})$ where $|r({\mathbf x}, {\mathbf h})|\leq C({\mathbf x}, {\mathbf h})\|{\mathbf h}\|^{|\alpha|}$ and $\lim_{{\mathbf h}\to 0}C({\mathbf x}, {\mathbf h})=0$.
\end{proposition}
For symmetric functions, $ \delta^{(\alpha,\alpha)}_{({\mathbf h},{\mathbf h}')}[F]({\mathbf x},{\mathbf x}) $ satisfies a finer property.
\begin{lemma}\label{finite-diff} Let $F \in C^{2k}(\boldsymbol{\Omega}\times \boldsymbol{\Omega})$ satisfy $F({\mathbf x}, {\mathbf y})=F({\mathbf y}, {\mathbf x})$. Then, for any $\alpha\in ({\mathbb N}\cup\{0\})^n: |\alpha|=k$, we have
$$\delta^{(\alpha,\alpha)}_{({\mathbf h},{\mathbf h}')}[F]({\mathbf x},{\mathbf x}) = \partial^{\alpha}_{\mathbf x}\partial^{\alpha}_{\mathbf y}F|_{({\mathbf x},{\mathbf x})}{\mathbf h}^{\alpha} ({\mathbf h}')^{\alpha} + r({\mathbf x}, {\mathbf h},{\mathbf h}'),$$ where 
\begin{equation*}
\begin{split}
|r({\mathbf x}, {\mathbf h},{\mathbf h}')|\leq \\
C_1({\mathbf x},{\mathbf h},{\mathbf h}'){\mathbf h}^{\alpha}\|{\mathbf h}'\|^{|\alpha|}+ C_2({\mathbf x},{\mathbf h},{\mathbf h}')\|{\mathbf h}\|^{|\alpha|}{\mathbf h}'^{\alpha}+C({\mathbf x},{\mathbf h},{\mathbf h}')\|{\mathbf h}\|^{|\alpha|}\|{\mathbf h}'\|^{|\alpha|}
\end{split}
\end{equation*}
and $\lim_{({\mathbf h},{\mathbf h}')\to ({\mathbf 0},{\mathbf 0})} C_1({\mathbf x},{\mathbf h},{\mathbf h}')=0$, $\lim_{({\mathbf h},{\mathbf h}')\to ({\mathbf 0},{\mathbf 0})} C_2({\mathbf x},{\mathbf h},{\mathbf h}')=0$, \\ $\lim_{({\mathbf h},{\mathbf h}')\to ({\mathbf 0},{\mathbf 0})} C({\mathbf x},{\mathbf h},{\mathbf h}')=0$.
\end{lemma}
\begin{proof}
A symbol $f({\mathbf h},{\mathbf h}') = o (g({\mathbf h},{\mathbf h}'))$ denotes  $\lim_{({\mathbf h},{\mathbf h}')\to ({\mathbf 0},{\mathbf 0})}\frac{f({\mathbf h},{\mathbf h}')}{g({\mathbf h},{\mathbf h}')} = 0$. 
Let us denote
\begin{equation*}
\begin{split}
q({\mathbf h},{\mathbf h}') = F({\mathbf x}+{\mathbf h}, {\mathbf x}+{\mathbf h}')-\sum_{\eta,\gamma: |\eta|, |\gamma|\leq k}\frac{1}{\eta!\gamma!}\partial^\eta_{\mathbf x} \partial^\gamma_{\mathbf y}F|_{({\mathbf x}, {\mathbf x})}{\mathbf h}^\eta {\mathbf h}'^\gamma - \\
\sum_{\eta: |\eta|\leq k}\frac{(\partial^\eta_{\mathbf x} F|_{({\mathbf x}, {\mathbf x}+{\mathbf h}')}-\sum_{\gamma: |\gamma|\leq k}\frac{1}{\gamma!}\partial^\eta_{\mathbf x} \partial^\gamma_{\mathbf y}F|_{({\mathbf x}, {\mathbf x})}  {\mathbf h}'^\gamma)}{\eta!}{\mathbf h}^\eta - \\
\sum_{\gamma: |\gamma|\leq k}\frac{(\partial^\gamma_{{\mathbf y}} F|_{({\mathbf x}+{\mathbf h}, {\mathbf x})}-\sum_{\eta: |\eta|\leq k}\frac{1}{\eta!}\partial^\eta_{\mathbf x} \partial^\gamma_{\mathbf y}F|_{({\mathbf x}, {\mathbf x})}  {\mathbf h}^\eta)}{\gamma!}{\mathbf h}'^\gamma
\end{split}
\end{equation*}

We will prove that $q({\mathbf h},{\mathbf h}') = o(\|{\mathbf h}'\|^{k}\|{\mathbf h}\|^{k})$. 
First, note that $\partial^\alpha_{{\mathbf h}}q({\mathbf h},{\mathbf h}')$, for $|\alpha|\leq k$, reads as
\begin{equation*}
\begin{split}
\partial^\alpha_{{\mathbf h}} q({\mathbf h},{\mathbf h}') = \partial^\alpha_{{\mathbf h}} F|_{({\mathbf x}+{\mathbf h}, {\mathbf x}+{\mathbf h}')}-\hspace{-20pt}\sum_{\eta ,\gamma : |\eta|\leq k-|\alpha|, |\gamma|\leq k}\partial^{\alpha+\eta}_{\mathbf x} \partial^{\gamma}_{\mathbf y}F|_{({\mathbf x}, {\mathbf x})} \frac{{\mathbf h}^{\eta} {\mathbf h}'^{\gamma}}{\eta!\gamma!} - \\
\sum_{\eta: |\eta|\leq k-|\alpha|}\frac{(\partial^{\alpha+\eta}_{\mathbf x}  F|_{({\mathbf x}, {\mathbf x}+{\mathbf h}')}-\sum_{\gamma: |\gamma|\leq k}\frac{1}{\gamma!}\partial^{\alpha+\eta}_{\mathbf x} \partial^{\gamma}_{\mathbf y}F|_{({\mathbf x}, {\mathbf x})}  {\mathbf h}'^{\gamma})}{\eta!}{\mathbf h}^{\eta} - \\
\sum_{\gamma: |\gamma|\leq k}\frac{(\partial^{\alpha}_{\mathbf x}\partial^\gamma_{{\mathbf y}} F|_{({\mathbf x}+{\mathbf h}, {\mathbf x})}-\sum_{\eta: |\eta|\leq k-|\alpha|}\frac{1}{\eta!}\partial^{\eta+\alpha}_{\mathbf x} \partial^\gamma_{\mathbf y}F|_{({\mathbf x}, {\mathbf x})}  {\mathbf h}^\eta)}{\gamma!}{\mathbf h}'^\gamma
\end{split}
\end{equation*}
and therefore,
\begin{equation*}
\begin{split}
\partial^\alpha_{{\mathbf h}} q({\mathbf 0},{\mathbf h}') = \partial^\alpha_{{\mathbf x}} F|_{({\mathbf x}, {\mathbf x}+{\mathbf h}')}-\sum_{\gamma :  |\gamma|\leq k}\partial^{\alpha}_{\mathbf x}\partial^{\gamma}_{\mathbf y}F|_{({\mathbf x}, {\mathbf x})} \frac{ {\mathbf h}'^{\gamma}}{\gamma!} - \\
(\partial^{\alpha}_{\mathbf x}  F|_{({\mathbf x}, {\mathbf x}+{\mathbf h}')}-\sum_{\gamma: |\gamma|\leq k}\frac{1}{\gamma!}\partial^{\alpha}_{\mathbf x} \partial^{\gamma}_{\mathbf y}F|_{({\mathbf x}, {\mathbf x})}  {\mathbf h}'^{\gamma})=0
\end{split}
\end{equation*}
Using $q(\cdot, {\mathbf h}')\in C^k(\boldsymbol{\Omega})$ and Taylor's expansion around ${\mathbf h}={\mathbf 0}$, we obtain
\begin{equation*}
\begin{split}
q({\mathbf h},{\mathbf h}') = \sum_{\alpha: |\alpha|=k}\partial^\alpha_{{\mathbf h}} q(\chi{\mathbf h},{\mathbf h}')\frac{{\mathbf h}^\alpha}{\alpha!}
\end{split}
\end{equation*}
where $\chi\in (0,1)$.

For $|\alpha|=k$, we have
\begin{equation*}
\begin{split}
\partial^\alpha_{{\mathbf h}} q(\chi {\mathbf h},{\mathbf h}') = \partial^\alpha_{{\mathbf x}} F|_{({\mathbf x}+\chi{\mathbf h}, {\mathbf x}+{\mathbf h}')}-\sum_{\gamma :  |\gamma|\leq k}\partial^{\alpha}_{\mathbf x} \partial^{\gamma}_{\mathbf y}F|_{({\mathbf x}, {\mathbf x})} \frac{ {\mathbf h}'^{\gamma}}{\gamma!} - \\
(\partial^{\alpha}_{\mathbf x}  F|_{({\mathbf x}, {\mathbf x}+{\mathbf h}')}-\sum_{\gamma: |\gamma|\leq k}\frac{1}{\gamma!}\partial^{\alpha}_{\mathbf x} \partial^{\gamma}_{\mathbf y}F|_{({\mathbf x}, {\mathbf x})}  {\mathbf h}'^{\gamma}) - \\
\sum_{\gamma: |\gamma|\leq k}\frac{(\partial^{\alpha}_{\mathbf x}\partial^\gamma_{{\mathbf y}} F|_{({\mathbf x}+\chi{\mathbf h}, {\mathbf x})}-\partial^{\alpha}_{\mathbf x} \partial^\gamma_{\mathbf y}F|_{({\mathbf x}, {\mathbf x})}  )}{\gamma!}{\mathbf h}'^\gamma = \\
\partial^\alpha_{{\mathbf x}} F|_{({\mathbf x}+\chi{\mathbf h}, {\mathbf x}+{\mathbf h}')}-\partial^{\alpha}_{\mathbf x}  F|_{({\mathbf x}, {\mathbf x}+{\mathbf h}')}
- \sum_{\gamma: |\gamma|\leq k}\frac{(\partial^{\alpha}_{\mathbf x}\partial^\gamma_{{\mathbf y}} F|_{({\mathbf x}+\chi{\mathbf h}, {\mathbf x})}-\partial^{\alpha}_{\mathbf x} \partial^\gamma_{\mathbf y}F|_{({\mathbf x}, {\mathbf x})}  )}{\gamma!}{\mathbf h}'^\gamma
\end{split}
\end{equation*}
If we denote $R({\mathbf h}')=\partial^\alpha_{{\mathbf x}} F|_{({\mathbf x}+\chi{\mathbf h}, {\mathbf x}+{\mathbf h}')}-\partial^{\alpha}_{\mathbf x}  F|_{({\mathbf x}, {\mathbf x}+{\mathbf h}')}$, then, by  Taylor's expansion theorem, we have $R({\mathbf h}')-\sum_{\gamma: |\gamma|\leq k}\frac{\partial^\gamma_{{\mathbf h}'} R({\mathbf 0}){\mathbf h}'^\gamma}{\gamma!}=o(\|{\mathbf h}'\|^k)$. The latter expression for $\partial^\alpha_{{\mathbf h}} q(\chi{\mathbf h},{\mathbf h}')$ exactly equals $R({\mathbf h}')-\sum_{\gamma: |\gamma|\leq k}\frac{\partial^\gamma_{{\mathbf h}'} R({\mathbf 0}){\mathbf h}'^\gamma}{\gamma!}$ and
 we conclude $$q({\mathbf h},{\mathbf h}')=\sum_{\alpha: |\alpha|=k}\partial^\alpha_{{\mathbf h}} q(\chi{\mathbf h},{\mathbf h}')\frac{{\mathbf h}^\alpha}{\alpha!}=o(\|{\mathbf h}\|^{k}\|{\mathbf h}'\|^{k}). $$ 

Thus, we proved that
\begin{equation}\label{LLL}
\begin{split}
F({\mathbf x}+{\mathbf h},{\mathbf x}+{\mathbf h}') =\sum_{\eta,\gamma: |\eta|,|\gamma|\leq k} A_{\eta, \gamma} {\mathbf h}^{\eta}({\mathbf h}')^{\gamma} + \sum_{\gamma: |\gamma|\leq k}  a_\gamma ({\mathbf h}) {\mathbf h}'^{\gamma}+\\
\sum_{\eta: |\eta|\leq k}  a_\eta ({\mathbf h}') {\mathbf h}^{\eta} +q({\mathbf h},{\mathbf h}'),
\end{split}
\end{equation}
where $A_{\eta, \gamma}=\frac{1}{\eta!\gamma!}\partial^\eta_{\mathbf x} \partial^\gamma_{\mathbf y}F|_{({\mathbf x}, {\mathbf x})}$, $a_\gamma ({\mathbf h}) =\partial^\gamma_{{\mathbf y}} F|_{({\mathbf x}+{\mathbf h}, {\mathbf x})}-\sum_{\eta: |\eta|\leq k}\frac{1}{\eta!}\partial^\eta_{\mathbf x} \partial^\gamma_{\mathbf y}F|_{({\mathbf x}, {\mathbf x})}  {\mathbf h}^\eta= o(\|{\mathbf h}\|^k)$ and $q({\mathbf h},{\mathbf h}')=o(\|{\mathbf h}\|^k\|{\mathbf h}'\|^k)$.

After plugging in the expression~\eqref{LLL} into~\eqref{2nd}, we have ($\odot$ denotes the Hadamard product)
\begin{equation*}
\begin{split}
\delta^{(\alpha,\alpha)}_{({\mathbf h},{\mathbf h}')}[F]({\mathbf x},{\mathbf x})  = \sum_{\beta,\beta': \beta\leq \alpha,\beta'\leq \alpha} (-1)^{|\beta|+|\beta'|} {\alpha \choose \beta}{\alpha \choose \beta'}\\
\sum_{\eta,\gamma: |\eta|,|\gamma|\leq k}\frac{1}{\eta!\gamma!}\partial^\eta_{\mathbf x} \partial^\gamma_{\mathbf y}F({\mathbf x}, {\mathbf y})|_{{\mathbf y}={\mathbf x}}  {\mathbf h}^{\eta}({\mathbf h}')^{\gamma} \beta^\eta (\beta')^\gamma+r({\mathbf x}, \beta\odot {\mathbf h},\beta'\odot {\mathbf h}') = \\
\sum_{\eta,\gamma: |\eta|,|\gamma|\leq k}\frac{1}{\eta!\gamma!}\partial^\eta_{\mathbf x} \partial^\gamma_{\mathbf y}F({\mathbf x}, {\mathbf y})|_{{\mathbf y}={\mathbf x}}  {\mathbf h}^{\eta}({\mathbf h}')^{\gamma} \hspace{-20pt}
\sum_{\beta,\beta': \beta\leq \alpha,\beta'\leq \alpha} 
 (-1)^{|\beta|+|\beta'|} {\alpha \choose \beta}{\alpha \choose \beta'} \beta^\eta (\beta')^\gamma+ \\
\sum_{\gamma: |\gamma|\leq k}  
\sum_{\beta,\beta': \beta\leq \alpha,\beta'\leq \alpha} 
 (-1)^{|\beta|+|\beta'|} {\alpha \choose \beta}{\alpha \choose \beta'}  a_{\gamma} (\beta\odot {\mathbf h})\beta'^{\gamma} {\mathbf h}'^{\gamma}+\\
\sum_{\eta: |\eta|\leq k} \sum_{\beta,\beta': \beta\leq \alpha,\beta'\leq \alpha} 
 (-1)^{|\beta|+|\beta'|} {\alpha \choose \beta}{\alpha \choose \beta'}    a_\eta (\beta'\odot {\mathbf h}') \beta^\eta {\mathbf h}^{\eta}+
o( \|{\mathbf h}\|^k \|{\mathbf h}'\|^k).
\end{split}
\end{equation*}
Note that
\begin{equation*}
\begin{split}
\sum_{\beta,\beta': \beta\leq \alpha,\beta'\leq \alpha} 
 (-1)^{|\beta|+|\beta'|} {\alpha \choose \beta}{\alpha \choose \beta'} \beta^\eta (\beta')^\gamma =\\
\prod_{i=1}^n \sum_{\beta_i=0}^{\alpha_i}\sum_{\beta'_i=0}^{\alpha_i} 
 (-1)^{\beta_i+\beta'_i}{\alpha_i \choose \beta_i}{\alpha_i \choose \beta'_i} \beta_i^{\eta_i} (\beta'_i)^{\gamma_i} =\\
\prod_{i=1}^n \big(\sum_{\beta_i=0}^{\alpha_i} 
 (-1)^{\alpha_i-\beta_i}{\alpha_i \choose \beta_i}\beta_i^{\eta_i}\big)  \big(\sum_{\beta'_i=0}^{\alpha_i} 
 (-1)^{\alpha_i-\beta'_i}{\alpha_i \choose \beta'_i}(\beta'_i)^{\gamma_i}\big) = \\
 \prod_{i=1}^n \delta^{\alpha_i}_h[x^{\eta_i}](0) \delta^{\alpha_i}_h[x^{\gamma_i}](0)
\end{split}
\end{equation*}
The expression that is in the RHS is just a finite difference of order $\alpha_i$ of $f(x) = x^{\eta_i}$ (or, $f(x) = x^{\gamma_i}$) for $h = 1$, due to
$\delta^{\alpha_i}_h[x^{\eta_i}](0) = \sum_{\beta_i=0}^{\alpha_i} 
 (-1)^{\alpha_i-\beta_i}{\alpha_i \choose \beta_i}\beta_i^{\eta_i}$.
It is well-known that $\delta^{\alpha_i}_h[x^{\eta_i}](x) = 0$, if $\eta_i<\alpha_i$ and $\delta^{\alpha_i}_h[x^{\eta_i}](x) = \eta_i!$, if $\eta_i=\alpha_i$. Thus, we have
\begin{equation*}
\begin{split}
\sum_{\beta,\beta': \beta\leq \alpha,\beta'\leq \alpha} 
 (-1)^{|\beta|+|\beta'|} {\alpha \choose \beta}{\alpha \choose \beta'} \beta^\eta (\beta')^\gamma = (\alpha!)^2[\eta=\gamma=\alpha]
\end{split}
\end{equation*}
and
\begin{equation*}
\begin{split}
\sum_{\beta: \beta\leq \alpha}
 (-1)^{|\beta|} {\alpha \choose \beta}\sum_{o: |o|\leq k} a_o (\beta'\odot {\mathbf h}') \beta^o {\mathbf h}^{o} = \alpha! a_\alpha (\beta'\odot {\mathbf h}')  {\mathbf h}^{\alpha}=o(\|{\mathbf h}'\|^k) {\mathbf h}^{\alpha}.
\end{split}
\end{equation*}
Therefore,
\begin{equation*}
\begin{split}
\delta^{(\alpha,\alpha)}_{({\mathbf h},{\mathbf h}')}[F]({\mathbf x},{\mathbf x})   = \frac{(\alpha!)^2}{(\alpha!)^2}\partial^\alpha_{\mathbf x} \partial^\alpha_{\mathbf y}K({\mathbf x}, {\mathbf y})|_{{\mathbf y}={\mathbf x}} {\mathbf h}^{\alpha} ({\mathbf h}')^{\alpha} 
+\\
o(\|{\mathbf h}'\|^k) {\mathbf h}^{\alpha}+o(\|{\mathbf h}\|^k) {\mathbf h}'^{\alpha}+o( \|{\mathbf h}\|^k \|{\mathbf h}'\|^k)
\end{split}
\end{equation*}
From the latter, the statement of Lemma directly follows.
\end{proof}

The following lemma is a direct consequence of Theorem 1 from~\cite{ZHOU2008456}. We give here its proof for the sake of completeness.
\begin{lemma}\label{in-rkhs} Let $K\in C^{2|\alpha|}(\boldsymbol{\Omega}\times \boldsymbol{\Omega})$ and ${\mathbf x}\in \boldsymbol{\Omega}$ be fixed. Let $\{\lambda_i\}$ be a multiset of all positive eigenvalues of ${\rm O}_K$ (counting multiplicities). Then, $\partial^\alpha_{\mathbf x}K({\mathbf x}, \cdot)\in \mathcal{H}_K$ and $\|\partial^\alpha_{\mathbf x}K({\mathbf x}, \cdot)\|^2_{\mathcal{H}_K} = D_\alpha({\mathbf x})=\sum_{i=1}^\infty \lambda_i(\partial^\alpha_{\mathbf x}\phi_i({\mathbf x}))^2$.
\end{lemma}
\begin{proof}
Let us choose some sequence $\{h_i\}_{i=1}^\infty$ such that $\lim_{i\to \infty} h_i = 0$ and let $$f_i({\mathbf y}) =\frac{ \delta^\alpha_{(h_i,\cdots, h_i)}[K({\mathbf z},{\mathbf y})]({\mathbf x})}{h_i^{|\alpha|}}\in \mathcal{H}_K$$ where the finite difference operator $\delta^\alpha_{\mathbf h}$ is applied onto the first argument. The inner product between $f_i$ and $f_j$ equals:
\begin{equation*}
\begin{split}
\langle f_i,f_j \rangle_{\mathcal{H}_K} = \frac{\delta^{(\alpha,\alpha)}_{(h_i,\cdots, h_i,h_j,\cdots, h_j)} [K]({\mathbf x},{\mathbf x})}{h_i^{|\alpha|}h_j^{|\alpha|}}
\end{split}
\end{equation*}
Therefore,
\begin{equation*}
\begin{split}
\| f_i-f_j \|^2_{\mathcal{H}_K} = \frac{\delta^{(\alpha,\alpha)}_{(h_i,\cdots, h_i)} [K]({\mathbf x},{\mathbf x})}{h_i^{2|\alpha|}}+\frac{\delta^{(\alpha,\alpha)}_{(h_j,\cdots, h_j)} [K]({\mathbf x},{\mathbf x})}{h_j^{2|\alpha|}}-\\
2\frac{\delta^{(\alpha,\alpha)}_{(h_i,\cdots, h_i,h_j,\cdots, h_j)} [K]({\mathbf x},{\mathbf x})}{h_i^{|\alpha|}h_j^{|\alpha|}}
\end{split}
\end{equation*}
From Lemma~\ref{finite-diff} we obtain that for any $\varepsilon>0$ there exists $N_\varepsilon>0$ such that $|\frac{\delta^{(\alpha,\alpha)}_{(h_i,\cdots, h_i)} [K]({\mathbf x},{\mathbf x})}{h_i^{2|\alpha|}}-D_\alpha({\mathbf x})|<\varepsilon$ and $|\frac{\delta^{(\alpha,\alpha)}_{(h_i,\cdots, h_i,h_j,\cdots, h_j)} [K]({\mathbf x},{\mathbf x})}{h_i^{|\alpha|}h_j^{|\alpha|}}-D_\alpha({\mathbf x})|<\varepsilon$ whenever $i>N_\varepsilon$, $j>N_\varepsilon$. Therefore, $\| f_i-f_j \|^2_{\mathcal{H}_K} \leq 4\varepsilon$ if $i>N_\varepsilon$, $j>N_\varepsilon$. The latter means that $\{f_i\}\subseteq \mathcal{H}_K$ is a Cauchy sequence. From the completeness of $\mathcal{H}_K$ we conclude that $f_i\to^{\mathcal{H}_K} f$ where $f\in \mathcal{H}_K$. From Proposition~\ref{smale} we conclude that $f_i$ uniformly converges to $f$. By construction, the pointwise limit of $\{f_i = \frac{ \delta^\alpha_{(h_i,\cdots, h_i)}[K({\mathbf z},{\mathbf y})]({\mathbf x})}{h_i^{|\alpha|}}\}$ is $\partial^\alpha_{\mathbf x} K({\mathbf x},\cdot)$. Therefore, $f_i\to^{\mathcal{H}_K} \partial^\alpha_{\mathbf x} K({\mathbf x},\cdot)$ and $\partial^\alpha_{\mathbf x} K({\mathbf x},\cdot)\in \mathcal{H}_K$.

Let $f_{\mathbf x}({\mathbf y}) = \delta^\alpha_{\mathbf h}[K({\mathbf z},{\mathbf y})]({\mathbf x})$ for ${\mathbf h} = (h,\cdots,h)$. In fact, we have just proved that $\lim_{h\to 0} \frac{\delta^\alpha_{\mathbf h}[K({\mathbf z},{\mathbf y})]({\mathbf x})}{h^{|\alpha|}} = \partial^\alpha_{\mathbf x} K({\mathbf x},\cdot)$ in $\mathcal{H}_K$.
According to Mercer's theorem, we have
$$
\lim_{N\to \infty} \sup_{{\mathbf z},{\mathbf y}\in  \boldsymbol{\Omega}} |K({\mathbf z}, {\mathbf y})-\sum_{i=1}^N \lambda_i\phi_i({\mathbf z})\phi_i({\mathbf y})| = 0.
$$
A sum of $k$ uniformly convergent function series equals a uniformly convergent series of the corresponding $k$-sums, i.e.
$$
\frac{f_{\mathbf x}({\mathbf y})}{h^{|\alpha|}} = \frac{\delta^\alpha_{\mathbf h}[\lim_{N\to \infty}\sum_{i=1}^N \lambda_i\phi_i({\mathbf z})\phi_i({\mathbf y})]({\mathbf x})}{h^{|\alpha|}} =\lim_{N\to \infty} \sum_{i=1}^N \lambda_i\frac{\delta^\alpha_{\mathbf h}[\phi_i]({\mathbf x})}{h^{|\alpha|}}\phi_i({\mathbf y})
$$
Therefore, $\frac{f_{\mathbf x}({\mathbf y})}{h^{|\alpha|}} =\sum_{i=1}^\infty \lambda_i\frac{\delta^\alpha_{\mathbf h}[\phi_i]({\mathbf x})}{h^{|\alpha|}}\phi_i({\mathbf y})$ and the latter convergence is uniform over ${\mathbf y}$.

Therefore, $\int \frac{f_{\mathbf x}({\mathbf y})}{h^{|\alpha|}} \phi_i({\mathbf y})d\mu({\mathbf y}) = \lambda_i\frac{\delta^\alpha_{\mathbf h}[\phi_i]({\mathbf x})}{h^{|\alpha|}}$. A uniform convergence of $ \frac{f_{\mathbf x}({\mathbf y})}{h^{|\alpha|}}$ to $ \partial^\alpha_{\mathbf x} K({\mathbf x},{\mathbf y})$ as $h\to 0$ implies $$\lambda_i\partial^\alpha_{\mathbf x}\phi_i({\mathbf x}) = \lim_{h\to 0}  \lambda_i\frac{\delta^\alpha_{\mathbf h}[\phi_i]({\mathbf x})}{h^{|\alpha|}} = \int  \partial^\alpha_{\mathbf x} K({\mathbf x},{\mathbf y}) \phi_i({\mathbf y})d\mu({\mathbf y}).$$ Since $\partial^\alpha_{\mathbf x} K({\mathbf x},\cdot)\in \mathcal{H}_K$, using Proposition~\ref{smale}, we conclude:
\begin{equation*}
\begin{split}
\|\partial^\alpha_{\mathbf x} K({\mathbf x},\cdot)\|^2_{\mathcal{H}_K} = \sum_{i=1}^\infty \frac{\lambda^2_i(\partial^\alpha_{\mathbf x}\phi_i({\mathbf x}))^2}{\lambda_i}
\end{split}
\end{equation*}
Since $\lim_{h\to 0}\langle \frac{f_{\mathbf x}}{h^{|\alpha|}}, \frac{f_{\mathbf x}}{h^{|\alpha|}}\rangle_{\mathcal{H}_K} = \lim_{h\to 0}\frac{\Delta^\alpha_{\mathbf h} [K]({\mathbf x})}{h^{2|\alpha|}} = D_\alpha({\mathbf x})$ we finally obtain
\begin{equation*}
\begin{split}
D_\alpha({\mathbf x}) = \|\partial^\alpha_{\mathbf x} K({\mathbf x},\cdot)\|^2_{\mathcal{H}_K} = \sum_{i=1}^\infty \lambda_i(\partial^\alpha_{\mathbf x}\phi_i({\mathbf x}))^2
\end{split}
\end{equation*}
\end{proof}
\begin{lemma}\label{expand} Let $K({\mathbf x}, {\mathbf y})\in C^{2m}(\boldsymbol{\Omega}\times \boldsymbol{\Omega})$ and $\{\lambda_i\}$ be a multiset of all positive eigenvalues of ${\rm O}_K$ (counting multiplicities).
 Then,
$$
\partial^\alpha_{\mathbf x}\partial^\beta_{\mathbf y} K({\mathbf x}, {\mathbf y})  = \sum_{i=1}^\infty \lambda_i\partial^\alpha_{\mathbf x}\phi_i({\mathbf x})\partial^\beta_{\mathbf y}\phi_i({\mathbf y})
$$
 for $|\alpha|\leq m$ and $|\beta|\leq m$. 
\end{lemma}
\begin{proof} Again, since $\lambda_i\phi_i({\mathbf x}) = \int_{ \boldsymbol{\Omega}}  K({\mathbf x}, {\mathbf y})\phi_i({\mathbf y})d{\mathbf y}$, we conclude $\lambda_i\partial^\alpha_{\mathbf x}\phi_i({\mathbf x}) = \int_{ \boldsymbol{\Omega}}   \partial^\alpha_{\mathbf x} K({\mathbf x}, {\mathbf y})\phi_i({\mathbf y})d{\mathbf y}\in C(\boldsymbol{\Omega})$ for $|\alpha|\leq m$. 
From Lemma~\ref{in-rkhs} and Dini's theorem we conclude that the series 
$$\sum_{i=1}^\infty \lambda_i|\partial^\alpha_{\mathbf x}\phi_i({\mathbf x})\partial^\beta_{\mathbf y}\phi_i({\mathbf y})|\leq \frac{1}{2}\sum_{i=1}^\infty \lambda_i(\partial^\alpha_{\mathbf x}\phi_i({\mathbf x}))^2+\lambda_i(\partial^\beta_{\mathbf y}\phi_i({\mathbf y}))^2$$
is absolutely and uniformly convergent. Therefore, we can differentiate the function series, and conclude
$$
\sum_{i=1}^\infty \lambda_i\partial^\alpha_{\mathbf x}\phi_i({\mathbf x})\partial^\beta_{\mathbf y}\phi_i({\mathbf y}) = \partial^\alpha_{\mathbf x}\partial^\beta_{\mathbf y} \big(\sum_{i=1}^\infty \lambda_i\phi_i({\mathbf x})\phi_i({\mathbf y})\big)=\partial^\alpha_{\mathbf x}\partial^\beta_{\mathbf y} K({\mathbf x}, {\mathbf y}).
$$ 
\end{proof}
Let us denote 
\begin{equation*}
\begin{split}
K_N({\mathbf x}, {\mathbf y}) =K({\mathbf x}, {\mathbf y})-\sum_{i=1}^N \lambda_i\phi_i({\mathbf x})\phi_i({\mathbf y})
\end{split}
\end{equation*}
and 
\begin{equation*}
\begin{split}
K^{\alpha, \beta}_N({\mathbf x}, {\mathbf y}) =\partial^{\alpha}_{\mathbf x}\partial^{\beta}_{\mathbf y} K({\mathbf x}, {\mathbf y})-\sum_{i=1}^N \lambda_i\partial^{\alpha}_{\mathbf x} \phi_i({\mathbf x})\partial^{\beta}_{\mathbf y} \phi_i({\mathbf y})
\end{split}
\end{equation*}
\begin{lemma}\label{Causchy} Let $K({\mathbf x}, {\mathbf y})\in C^{2m}(\boldsymbol{\Omega}\times \boldsymbol{\Omega})$ for compact $\boldsymbol{\Omega}\subseteq {\mathbb R}^n$ and $\{\lambda_i\}$ be a multiset of all positive eigenvalues of ${\rm O}_K$ (counting multiplicities). Then, for any $|\alpha|\leq m$, $|\beta|\leq m$, we have
$$
|K^{\alpha, \beta}_N({\mathbf x}, {\mathbf y}) | \leq K^{\alpha, \alpha}_N ( {\mathbf x}, {\mathbf x})^{1/2} K^{\beta, \beta}_N ( {\mathbf y}, {\mathbf y})^{1/2}\leq  D_\alpha ({\mathbf x})^{1/2}D_\beta ({\mathbf y})^{1/2}.
$$
\end{lemma}
\begin{proof}
From Lemmas~\ref{in-rkhs} and~\ref{expand} we have
$$
\partial^\alpha_{\mathbf x} K_N({\mathbf x}, \cdot) = \sum_{i=N+1}^\infty \lambda_i\partial^\alpha_{\mathbf x} \phi_i({\mathbf x})\phi_i({\mathbf y})\in {\mathcal H}_K.
$$
Using Proposition~\ref{smale}, and again, Lemma~\ref{expand}, we obtain
\begin{equation*}
\begin{split}
\partial^\alpha_{\mathbf x} \partial^\beta_{\mathbf y} K_N({\mathbf x}, {\mathbf y}) =  \sum_{i=N+1}^\infty \lambda_i\partial^\alpha_{\mathbf x} \phi_i({\mathbf x})\partial^\beta_{\mathbf y}\phi_i({\mathbf y})= \langle \partial^\alpha_{\mathbf x} K_N({\mathbf x}, \cdot), \partial^\beta_{\mathbf y} K_N({\mathbf y}, \cdot)\rangle_{\mathcal{H}_K} .
\end{split}
\end{equation*}
Finally, the Cauchy-Schwartz inequality gives us
\begin{equation*}
\begin{split}
 |\langle \partial^\alpha_{\mathbf x} K_N({\mathbf x}, \cdot), \partial^\beta_{\mathbf y} K_N({\mathbf y}, \cdot)\rangle_{\mathcal{H}_K} |\leq \| \partial^\alpha_{\mathbf x} K_N({\mathbf x}, \cdot)\|_{\mathcal{H}_K}\cdot \|\partial^\beta_{\mathbf y} K_N({\mathbf y}, \cdot)\|_{\mathcal{H}_K} = \\ K^{\alpha, \alpha}_N ( {\mathbf x}, {\mathbf x})^{1/2} K^{\beta, \beta}_N ( {\mathbf y}, {\mathbf y})^{1/2}.
\end{split}
\end{equation*}
Note that
\begin{equation*}
\begin{split}
 K^{\alpha,\alpha}_N( {\mathbf x}, {\mathbf y}) = \sum_{i=N+1}^\infty \lambda_i\partial^\alpha_{\mathbf x} \phi_i({\mathbf x})^2 \leq \sum_{i=1}^\infty \lambda_i\partial^\alpha_{\mathbf x} \phi_i({\mathbf x})^2 = D_\alpha ({\mathbf x}).
\end{split}
\end{equation*}
Therefore, we have
\begin{equation*}
\begin{split}
|K^{\alpha, \beta}_N({\mathbf x}, {\mathbf y}) | \leq  D_\alpha ({\mathbf x})^{1/2}D_\beta ({\mathbf y})^{1/2}.
\end{split}
\end{equation*}
\end{proof}

\begin{proof}[Proof of Theorem~\ref{GN-theorem}.]
 The tightness of our bounds strongly depends on the constant $ C_{\boldsymbol{\Omega}, p} $ in the Gagliardo-Nirenberg inequality, which reads as~\cite{Brezis,Nirenberg}
$$
\|u\|_{L_\infty( \boldsymbol{\Omega})}\leq C_{\boldsymbol{\Omega}, p} \|u\|^{1-\theta}_{L_1( \boldsymbol{\Omega})}\cdot \|D^m u\|^{\theta}_{L_p( \boldsymbol{\Omega})}+C_{\boldsymbol{\Omega}, p}\|u\|_{L_1( \boldsymbol{\Omega})},
$$
where $\theta(\frac{n}{p}-m)+(1-\theta)n=0$ and $ \|D^m u\|_{L_p( \boldsymbol{\Omega})} = \max_{\alpha: |\alpha|=m}\|\partial^\alpha_{{\mathbf x}} u({\mathbf x})\|_{L_p( \boldsymbol{\Omega})}$. Thus, $\theta=\frac{n}{n-n/p+m} = (1+\frac{m}{n}-\frac{1}{p})^{-1}$.

Using $\sup|K_N({\mathbf x}, {\mathbf y})|\leq \sup K_N({\mathbf x}, {\mathbf x})^{1/2} K_N({\mathbf y}, {\mathbf y})^{1/2} = \sup K_N({\mathbf x}, {\mathbf x})$ and the Gagliardo-Nirenberg inequality we have
\begin{equation*}
\begin{split}
C_K^N = \|K_N({\mathbf x}, {\mathbf y}) \|_{L_\infty(\boldsymbol{\Omega}\times \boldsymbol{\Omega})} = \|K_N({\mathbf x}, {\mathbf x}) \|_{L_\infty( \boldsymbol{\Omega})} \leq \\
C_{\boldsymbol{\Omega}, p}\|K_N({\mathbf x}, {\mathbf x}) \|^{1-\theta}_{L_1( \boldsymbol{\Omega})}\cdot \|D^m K_N({\mathbf x}, {\mathbf x}) \|^{\theta}_{L_p( \boldsymbol{\Omega})} +C_{\boldsymbol{\Omega}, p}\|K_N({\mathbf x}, {\mathbf x}) \|_{L_1( \boldsymbol{\Omega})}\leq\\ 
C_{\boldsymbol{\Omega}, p} \big(\sum_{i=N+1}^\infty\lambda_i\big)^{1-\theta}\cdot \|D^m K_N({\mathbf x}, {\mathbf x}) \|^{\theta}_{L_p( \boldsymbol{\Omega})} + C_{\boldsymbol{\Omega}, p}\sum_{i=N+1}^\infty\lambda_i.
\end{split}
\end{equation*}
Lemma~\ref{Causchy} gives us 
\begin{equation*}
\begin{split}
|\partial^\alpha_{{\mathbf x}} [K_N({\mathbf x}, {\mathbf x})]| = |\sum_{\beta\leq \alpha}{\alpha \choose \beta}K^{\beta, \alpha-\beta}_N({\mathbf x}, {\mathbf x})| \leq \sum_{\beta\leq \alpha}{\alpha \choose \beta} D_\beta ({\mathbf x})^{1/2}D_{\alpha-\beta} ({\mathbf x})^{1/2}.
\end{split}
\end{equation*} 
Therefore, we have
\begin{equation*}
\begin{split}
C_K^N\leq
C_{\boldsymbol{\Omega}, p} \big(\sum_{i=N+1}^\infty\lambda_i\big)^{1-\theta}\cdot \max_{\alpha: |\alpha|=m}(\sum_{\beta\leq \alpha}{\alpha \choose \beta} \|\sqrt{D_\beta D_{\alpha-\beta}}\|_{L_p(\boldsymbol{\Omega})})^\theta + \\
C_{\boldsymbol{\Omega}, p}\sum_{i=N+1}^\infty\lambda_i.
\end{split}
\end{equation*}
Theorem proved.
\end{proof}

\begin{proof}[Proof of Theorem~\ref{GN-theorem2}.]  Another version of the Gagliardo-Nirenberg inequality, now for the domain $\boldsymbol{\Omega}\times \boldsymbol{\Omega}$, is
$$
\|u\|_{L_\infty( \boldsymbol{\Omega}\times \boldsymbol{\Omega})}\leq D_{\boldsymbol{\Omega}, p} \|u\|^{1-\theta}_{L_2( \boldsymbol{\Omega}\times \boldsymbol{\Omega})}\cdot \|D^m u\|^{\theta}_{L_p( \boldsymbol{\Omega}\times \boldsymbol{\Omega})}+D_{\boldsymbol{\Omega}, p}\|u\|_{L_2( \boldsymbol{\Omega}\times \boldsymbol{\Omega})},
$$
where $\theta(\frac{n}{p}-m)+(1-\theta)\frac{n}{2}=0$. Therefore, $\theta = \frac{n/2}{n/2-n/p+m} = (1+\frac{2m}{n}-\frac{2}{p})^{-1}$.
For $u({\mathbf x}, {\mathbf y}) = K_N({\mathbf x}, {\mathbf y})$, we have
\begin{equation*}
\begin{split}
C_K^N = \|K_N({\mathbf x}, {\mathbf y}) \|_{L_\infty(\boldsymbol{\Omega}\times \boldsymbol{\Omega})} \leq \\ D_{\boldsymbol{\Omega}, p} \|K_N\|^{1-\theta}_{L_2( \boldsymbol{\Omega}\times \boldsymbol{\Omega})}\cdot \|D^m K_N\|^{\theta}_{L_p( \boldsymbol{\Omega}\times \boldsymbol{\Omega})}+D_{\boldsymbol{\Omega}, p}\|K_N\|_{L_2( \boldsymbol{\Omega}\times \boldsymbol{\Omega})}.
\end{split}
\end{equation*}
Using Lemma~\ref{Causchy} we obtain
\begin{equation*}
\begin{split}
\|D^m K_N\|_{L_p} = \max_{|\alpha|+|\beta|=m}\|\partial^\alpha_{{\mathbf x}}\partial^\beta_{{\mathbf y}} [K_N({\mathbf x}, {\mathbf y})]\|_{L_p} \leq \max_{|\alpha|+|\beta|=m} \|\sqrt{D_\alpha }\|_{L_p}\|\sqrt{D_{\beta}}\|_{L_p}.
\end{split}
\end{equation*} 
Therefore,
\begin{equation*}
\begin{split}
C_K^N \leq \\ D_{\boldsymbol{\Omega}, p} (\sum_{i=N+1}^\infty \lambda_i^2)^{(1-\theta)/2}\cdot  \max_{|\alpha|+|\beta|=m} \|\sqrt{D_\alpha }\|^\theta_{L_p}\|\sqrt{D_{\beta}}\|^\theta_{L_p}+D_{\boldsymbol{\Omega}, p}(\sum_{i=N+1}^\infty \lambda_i^2)^{1/2}.
\end{split}
\end{equation*}
Theorem proved.
\end{proof}

\section{Applications}
{\bf Bounding the kernel of ${\rm O}^\gamma_K$.} For $\gamma>0$, let us denote $$K^\gamma({\mathbf x}, {\mathbf y}) = \sum_{i=1}^\infty\lambda_i^{\gamma}\phi_i({\mathbf x})\phi_i({\mathbf y}).$$ In general, checking the condition  $\sup_{{\mathbf x}\in \boldsymbol{\Omega}} K^{\gamma}({\mathbf x},{\mathbf x})<\infty$ requires the study of eigenvectors $\phi_i$. For kernels that appear in applications~\cite{cucker_zhou_2007,A2020_4_263}, a concrete form of eigenvectors is known only in few cases~\cite{Rosasco}. In the current paper we are interested in information that can be extracted from a behavior of eigenvalues $\{\lambda_i\}$. Let us formulate one example of such a sufficient condition.

 Note that if $\sum_{i=1}^\infty \lambda_i^{2\gamma}<\infty$, then $\sum_{i=1}^\infty\lambda_i^{\gamma}\phi_i({\mathbf x})\phi_i({\mathbf y})\in L_2(\boldsymbol{\Omega}\times\boldsymbol{\Omega})$. In a special case $\gamma=\frac{1}{2}$ we have $\sum_{i=1}^\infty \lambda_i^{2\gamma} = {\rm Tr}({\rm O}_{K})<\infty$. Therefore, $K^\gamma\in L_2(\boldsymbol{\Omega}\times\boldsymbol{\Omega})$ for $\gamma\in [\frac{1}{2},1]$. The boundedness of $K^\gamma$ on the diagonal, i.e. $\sup_{{\mathbf x}\in \boldsymbol{\Omega}} K^{\gamma}({\mathbf x},{\mathbf x})<\infty$ is equivalent to $K^\gamma\in C(\boldsymbol{\Omega}\times\boldsymbol{\Omega})$. Indeed, if $K^{\gamma}({\mathbf x},{\mathbf x})<C$, then  $f_N({\mathbf x}) = \sum_{i=1}^N\lambda_i^{\gamma}\phi_i({\mathbf x})^2$ is a monotonically increasing sequence of nonnegative continuous functions on a compact set $\boldsymbol{\Omega}$, bounded by $C$. Then, by monotone convergence theorem, $\{f_N\}$ uniformly converges to a continuous function $K^{\gamma}({\mathbf x},{\mathbf x})$. From the uniform convergence of the series $\sum_{i=1}^\infty\lambda_i^{\gamma}\phi_i({\mathbf x})^2$ it is straightforward that $\sum_{i=1}^\infty\lambda_i^{\gamma}\phi_i({\mathbf x})\phi_i({\mathbf y})<\frac{1}{2}\sum_{i=1}^\infty\lambda_i^{\gamma}(\phi_i({\mathbf x})^2+\phi_i({\mathbf y})^2)$ is also uniformly convergent to a continuous function.

\begin{theorem} Let $K\in C^{2m}(\boldsymbol{\Omega}\times \boldsymbol{\Omega})$ and $\gamma\in (0,1)$. Then, for $\sup_{{\mathbf x}\in \boldsymbol{\Omega}} K^{1-\gamma}({\mathbf x},{\mathbf x})<\infty$ it is sufficient to have $$\sum_{i=N+1}^\infty\lambda_i^2 = o(\lambda_{N}^{\frac{(2m+n)\gamma}{m}})$$ and $$\sum_{N=1}^\infty (\sum_{i=N+1}^\infty\lambda_i^2)^{\frac{m}{2m+n}}(\lambda_{N+1}^{-\gamma}-\lambda_{N}^{-\gamma})<\infty.$$
\end{theorem}
\begin{proof}
Let us denote $K_N({\mathbf x}) = \sum_{i=N}^\infty\lambda_i\phi_i({\mathbf x})^2$.
\begin{equation*}
\begin{split}
K^{1-\gamma}({\mathbf x}, {\mathbf x}) =  \sum_{i=1}^\infty\lambda_i^{-\gamma}\lambda_i\phi_i({\mathbf x})^2 =  
\sum_{i=1}^\infty\lambda_i^{-\gamma} (K_i({\mathbf x})-K_{i+1}({\mathbf x})) = \\
{\rm\,\,using\,\,summation\,\,by\,\,parts\,\,formula} \\
=\lambda_1^{-\gamma}K_1({\mathbf x})-\lim_{N\to+\infty}\lambda_{N}^{-\gamma}K_{N+1}({\mathbf x})+
\sum_{i=2}^\infty K_i({\mathbf x})(\lambda_i^{-\gamma}-\lambda_{i-1}^{-\gamma})\leq \\
\lambda_1^{-\gamma}D^2_K+\lim_{N\to+\infty}\lambda_{N}^{-\gamma}(C^{N}_K)^2+
\sum_{N=1}^\infty (C^N_K)^2(\lambda_{N+1}^{-\gamma}-\lambda_{N}^{-\gamma})
\end{split}
\end{equation*}
In Theorem~\ref{GN-theorem2} it was shown that for $K\in C^{2m}(\boldsymbol{\Omega}\times \boldsymbol{\Omega})$ we have $(C^N_K)^2\leq C(\sum_{i=N+1}^\infty\lambda_i^2)^{\frac{m}{2m+n}}$. Therefore, $\lim_{N\to \infty}\lambda_{N}^{-\gamma}(\sum_{i=N+1}^\infty\lambda_i^2)^{\frac{m}{2m+n}}=0$ and $$\sum_{N=1}^\infty (\sum_{i=N+1}^\infty\lambda_i^2)^{\frac{m}{2m+n}}(\lambda_{N+1}^{-\gamma}-\lambda_{N}^{-\gamma})<\infty$$ is sufficient for $\sup_{{\mathbf x}\in \boldsymbol{\Omega}} K^{1-\gamma}({\mathbf x},{\mathbf x})<\infty$.
\end{proof}

Let us show how to apply the latter bound for infinitely differentiable kernels. In the case of an infinitely differentiable kernel, we have $$(C^N_K)^2\leq C(\sum_{i=N+1}^\infty\lambda_i^2)^{0.5-\varepsilon}$$ for any $\varepsilon>0$. Let us additionally assume that eigenvalues of ${\rm O}_K$ are rapidly vanishing, i.e. $\sum_{i=N+1}^\infty\lambda_i^2 = \mathcal{O}(\lambda^2_{N+1})$ and $\sum_{i=1}^\infty\lambda_i^\varepsilon < \infty$ for any $\varepsilon>0$. Note that these conditions are satisfied for the Gaussian kernel on a box or a ball in ${\mathbb R}^n$, analytic kernels on a finite interval~\cite{Little}. Let $\gamma\in (0, 1)$. 
We have $\lambda_{N}^{-\gamma}(C^{N}_K)^2\leq C\lambda_{N}^{-\gamma}(\sum_{i=N+1}^\infty\lambda_i^2)^{0.5-\varepsilon}=\mathcal{O}(\lambda_{N+1}^{1-2\varepsilon-\gamma})\mathop\to\limits^{N\to\infty}0$, since $\varepsilon$ can be chosen to satisfy $\gamma<1-2\varepsilon$.
Also,  $\sum_{N=1}^\infty (C^N_K)^2(\lambda_{N+1}^{-\gamma}-\lambda_{N}^{-\gamma})\leq C\sum_{N=1}^\infty (\sum_{i=N+1}^\infty\lambda_i^2)^{0.5-\varepsilon}\lambda_{N+1}^{-\gamma}\leq C'\sum_{N=1}^\infty \lambda_{N+1}^{1-2\varepsilon-\gamma} < \infty$. Thus, for $\gamma\in (0, 1]$, $K^\gamma$ is bounded and continuous.

{\bf Bounding the supremum norm of eigenvectors.} The condition $$\sup_{N}\|\phi_{N}\|_{L_\infty(\boldsymbol{\Omega})}<\infty$$ is popular in various statements concerning Mercer kernels, though it is believed to be hard to check. Discussions of that issue can be found in~\cite{Zhou,Regularizationkernel,Steinwart2012}. 

Since $\lambda_{N+1}\phi_{N+1}({\mathbf x})^2\leq K({\mathbf x},{\mathbf x})-\sum_{i=1}^N \lambda_i \phi_i({\mathbf x})^2$, we conclude
\begin{equation*}
\begin{split}
\|\phi_{N+1}\|_{L_\infty(\boldsymbol{\Omega})}\leq \lambda^{-1/2}_{N+1} \sqrt{C^N_K}
\end{split}
\end{equation*}
Thus, any upper bound for $C^N_K$ leads to an upper bound of $\|\phi_{N+1}\|_{L_\infty(\boldsymbol{\Omega})}$. For a uniform boundedness of $\|\phi_{N+1}\|_{L_\infty(\boldsymbol{\Omega})}$ we need $C^N_K = \mathcal{O}(\lambda_{N+1})$. Unfortunately, RHS of our bounds are not $\mathcal{O}(\lambda_{N+1})$, though they can be used to show a moderate growth rate of $\|\phi_{N+1}\|_{L_\infty(\boldsymbol{\Omega})}$.

\bibliographystyle{acm}
\newcommand{\noopsort}[1]{}

\end{document}